\documentclass{article} 
\usepackage{nips15submit_e,times}
\usepackage{hyperref}
\usepackage{url}

\usepackage[square,numbers]{natbib}
\usepackage{amsmath}
\usepackage{amstext}
\usepackage{amssymb}
\usepackage{amsthm}
\usepackage{mathtools}

\usepackage{graphicx}
\usepackage{placeins}

\newcommand{\ev}{\mathbb{E}}

\newtheorem{definition}{Definition}
\newtheorem{lemma}{Lemma}
\newtheorem{Theorem}{Theorem}

\newtheorem{test}{Test}
\newtheorem{proposition}{Proposition}

\title{Fast Two-Sample Testing with Analytic Representations of Probability Measures}

\author{
Kacper Chwialkowski\\
Computer Science Department, Gatsby Computational Neuroscience Unit\\
University College London, \\
\texttt{kacper.chwialkowski@gmail.com} \\
\And
Aaditya Ramdas \\
Machine Learning and Statistics School of Computer Science \\
Carnegie Mellon University \\
\texttt{aramdas@cs.cmu.edu} \\
\And
Dino Sejdinovic \\
Department of Statistics \\
University of Oxford \\
\texttt{dino.sejdinovic@gmail.com} \\
\And
Arthur Gretton\\
Gatsby Computational Neuroscience Unit \\
University College London \\
\texttt{arthur.gretton@gmail.com} \\
}

\nipsfinalcopy 

\begin{document}
\setlength{\bibsep}{0pt plus 0.5ex}

\maketitle

\vspace{-2mm}
\begin{abstract}
We propose a class of nonparametric two-sample tests with a cost linear in the sample size. Two tests are given, both  based on an ensemble of distances between analytic functions representing each of the distributions. The first test uses smoothed empirical characteristic functions to represent the distributions, the second uses distribution embeddings in a reproducing kernel Hilbert space. Analyticity implies that differences in the distributions may be detected almost surely at a finite number of randomly chosen locations/frequencies. The new tests are consistent against a larger class of alternatives than the previous linear-time tests based on the (non-smoothed) empirical characteristic functions, while being much faster than the current state-of-the-art quadratic-time kernel-based or energy distance-based tests. Experiments on artificial benchmarks and on challenging real-world testing problems demonstrate that our tests give a better power/time tradeoff than  competing approaches, and in some cases, better outright power than even the most expensive quadratic-time tests. This performance advantage is retained even in high dimensions, and in cases where the difference in distributions is not observable with low order statistics. 
\end{abstract}
\vspace{-2mm}
\section{Introduction}

Testing whether two random variables are identically distributed without imposing any parametric assumptions on their distributions is important in a variety of scientific applications. These include data integration in bioinformatics \cite{borgwardt2006integrating}, benchmarking for steganography \cite{pevny2008benchmarking} and automated model checking \cite{Lloyd2014}. Such problems are addressed in the statistics literature via two-sample tests (also known as homogeneity tests).

Traditional approaches to two-sample testing are based on  distances between representations of the distributions, such as density functions, cumulative distribution functions, characteristic functions or mean embeddings in a reproducing kernel Hilbert space (RKHS) \cite{SriGreFukLanetal10, Sriperumbudur2011}. These representations are infinite dimensional objects, which poses challenges when defining a distance between distributions. Examples of such distances include the  classical  Kolmogorov-Smirnov distance (sup-norm between cumulative distribution functions);  the Maximum Mean Discrepancy (MMD) \cite{Gretton2012}, an RKHS norm of the difference between mean embeddings, and the $\mathbb{N}$-distance (also known as energy distance) \cite{zinger1992characterization,szekely2003statistics,baringhaus2004new},
 which is an MMD-based test for a particular family of kernels \cite{SejSriGreFuk13} .
  Tests may also be based on quantities other  than distances, an example being the Kernel Fisher Discriminant (KFD) \cite{Harchaoui07}, the estimation of which still requires calculating the RKHS norm of a difference of mean embeddings, with normalization by an inverse covariance operator.

 In contrast to consistent two-sample tests, heuristics based on pseudo-distances, such as the difference between characteristic functions  evaluated at a single frequency, have been studied in the context of goodness-of-fit tests \cite{heathcote1972test, heathcote1977integrated}. It was shown that the power of such tests can be maximized against fully specified alternative hypotheses, where test power is the probability of correctly rejecting the null hypothesis that the distributions are the same.  In other words, if the class of distributions being distinguished is known in advance, then the tests can focus only at those particular frequencies where the characteristic functions differ most. This approach was generalized to evaluating the empirical characteristic functions at multiple distinct frequencies by \cite{EppsSingleton1986}, thus improving on tests that need  to know the single ``best'' frequency in advance (the cost remains linear in the sample size, albeit with a larger constant). This approach still fails to solve the consistency problem, however: two distinct characteristic functions can agree on an interval, and if the tested frequencies fall in that interval, the distributions will be indistinguishable.

 
 In Section \ref{sec:distances} of the present work, we 
 introduce two novel distances between distributions, which both  use a parsimonious representation of the probability measures.  The first distance builds on the notion of differences in characteristic functions with the introduction of \textit{smooth  characteristic functions},  which can be though as the analytic analogues of the characteristics functions. A distance between  smooth characteristic functions evaluated at a single random frequency is almost surely a distance (Definition \ref{rand:metric} formalizes this  concept) between these two distributions. In other words, there is no need to calculate the whole infinite dimensional representation - it is almost surely sufficient to evaluate it at a single random frequency (although checking more frequencies will generally result in more powerful tests). The second distance is based on analytic mean embeddings of two distributions in a characteristic RKHS; again, it is sufficient to evaluate the distance between mean embeddings at a single randomly chosen point to obtain almost surely a distance.
 To our knowledge, this representation is the first mapping of the space of probability measures into a finite dimensional Euclidean space (in the simplest case, the real line) that is almost surely an injection, and as a result almost surely a metrization. This metrization  is very appealing from a computational viewpoint, since the statistics based on it have linear time complexity 
 (in the number of samples) and constant memory requirements.    

 We construct statistical tests in Section \ref{sec:test}, based on empirical estimates of differences
 in the analytic representations of the two distributions.
 Our tests have a number of theoretical and computational advantages over previous approaches. The test based on differences between analytic mean embeddings  is a.s. consistent for all distributions, and the  test based on differences between smoothed characteristic functions is a.s. consistent for all distributions with integrable characteristic functions (contrast with [7], which is only consistent under much more onerous conditions, as discussed above). This same weakness was used by  \cite{AlbaFernandez2008} in justifying a test that integrates over the {\em entire} frequency domain (albeit at cost quadratic in the sample size), for which the  quadratic-time MMD is a generalization \cite{Gretton2012}. Compared with such quadratic time  tests, our tests can be conducted in linear time -- hence, we expect their power/computation tradeoff to be superior.

 We provide several  experimental benchmarks (Section \ref{sec:experiments}) for our tests.  First, we compare test power as a function of computation time for two real-life testing settings: amplitude modulated audio samples, and the Higgs dataset, which are both challenging multivariate testing problems.
 Our tests give a better power/computation tradeoff than the characteristic function-based tests of \cite{EppsSingleton1986}, the previous sub-quadratic-time MMD tests \cite{GreSriSejetal12,ZarGreBla13}, and the quadratic-time MMD test.
 In terms of power when unlimited computation time is available, we might expect worse performance
 for the new tests, in line with findings for linear- and sub-quadratic-time MMD-based tests \cite{HoShieh06,Gretton2012,GreSriSejetal12,ZarGreBla13}. Remarkably, such a loss of power is not the rule: for instance,  when  distinguishing signatures of the Higgs boson from background noise \cite{Baldi2014} ('Higgs dataset'), we observe that a test based on differences in smoothed empirical characteristic functions outperforms the quadratic-time MMD. This is in contrast to linear- and sub-quadratic-time MMD-based tests, which by construction are less powerful than  quadratic-time MMD.
 Next, for challenging artificial data (both high-dimensional distributions, and distributions for which the difference is very subtle), our tests again give a better power/computation tradeoff than  competing methods.

\vspace{-2mm}
\section{Analytic embeddings and distances}\label{sec:distances}
In this section we consider  mappings from the space of probability measures into a sub-space of real valued analytic functions. We will show that evaluating  these maps at $J$ randomly selected points is almost surely injective for any $J>0$. Using this result, we obtain a simple (randomized) metrization of the space of probability measures. This metrization is used in the next section to construct linear-time nonparametric two-sample tests.    


To motivate our approach, we begin by recalling an integral family of distances between distributions, denoted Maximum Mean Discrepancies (MMD) \cite{Gretton2012}. The MMD is defined as
\begin{equation}
\text{MMD}(P,Q) = \sup_{f \in B_k} \left[\int_{E} f dP - \int_{E} f dQ \right],
\end{equation}
where $P$ and $Q$ are probability measures on $E$, and $B_k$ is the unit ball in the RKHS $H_k$ associated with a positive definite kernel $k:E \times E\to\mathbf R$. A popular choice of $k$ is the Gaussian kernel $k(x,y) = \exp(-\|x-y\|^2/\gamma^2)$ with bandwidth parameter $\gamma$. It can be shown that the MMD is equal to the RKHS distance between so called mean embeddings,
\begin{equation}\label{eq:MMD-meanembedding}
 \text{MMD}(P,Q) = \| \mu_P - \mu_Q \|_{H_k},
\end{equation}
where $\mu_P$ is an embedding of the probability measure $P$  to $H_k$,
\begin{equation}
 \mu_{P}(t)= \int_{E} k(x,t) dP(x),
\end{equation}
and $\| \cdot \|_{H_k}$ denotes the norm in the RKHS $H_k$. When $k$ is translation invariant, i.e., $k\left(x,y\right)=\kappa(x-y)$,  the squared MMD can be written \cite[Corollary 4]{SriGreFukLanetal10} as
\begin{equation}\label{eq:mmd_bochner}
\text{MMD}^2(P,Q) = \int_{\mathbf R^d} \left| \varphi_P(t) - \varphi_Q(t) \right|^2 F^{-1}\kappa(t) dt,
\end{equation}
where $F$ denotes the Fourier transform, $F^{-1}$ is the inverse Fourier transform,\ and $\varphi_P$, $\varphi_Q$ are the characteristic functions of $P$, $Q$, respectively. \ From \cite[Theorem 9]{SriGreFukLanetal10}, a kernel  is called {\em characteristic} when
\begin{equation}\label{eq:characteristic}
\text{MMD}(P,Q) = 0 \mbox{ iff } P=Q.
\end{equation}
Any bounded, continuous, translation-invariant kernel whose inverse Fourier transform is almost everywhere non-zero is characteristic \cite{SriGreFukLanetal10}. By representation \eqref{eq:MMD-meanembedding}, it is clear that MMD with a characteristic kernel is a metric.


\paragraph{Pseudometrics based on characteristic functions.}
A practical limitation when using the MMD in testing is that empirical estimates are expensive
to compute, these being the sum of two U-statistics and an empirical average, with cost quadratic
in the sample size. We might instead consider a
 finite dimensional approximation to the MMD, achieved by estimating the integral \eqref{eq:mmd_bochner}, with the random variable 
\begin{equation}\label{eq:unsmoothedCharFuncDist}
d_{\varphi,J}^2(P,Q) =  \frac{1}{J} \sum_{j=1}^{J} | \varphi_P(T_j) - \varphi_Q(T_j) |^2 , 
\end{equation}
where $\left\{T_j\right\}_{j=1}^J$ are sampled independently from the distribution with a density function $F^{-1}\kappa$. This type of approximation is applied to various kernel algorithms under the name of {\em random Fourier features} \cite{Rahimi2007, LeSarlosSmola2013}.  In the statistical testing literature, the quantity $d_{\varphi,J}(P,Q)$ predates the MMD by a considerable time, and was studied in \cite{heathcote1972test, heathcote1977integrated,EppsSingleton1986}, and more recently revisited in \cite{zhao2014fastmmd}. Our first proposition is that $d_{\varphi,J}^2(P,Q)$ can be a poor choice of distance between probability measures, as it fails to distinguish a large class of measures.  The following result is proved in the Appendix.
\begin{proposition}
\label{prop:notSufficient}
 Let $J\in\mathbb N$ and let   $\left\{T_j\right\}_{j=1}^J$ be a sequence of real valued i.i.d. random variables with a distribution which is absolutely continuous with respect to the Lebesgue measure. For any $\epsilon>0$, there exists an uncountable set $\mathcal A$ of mutually distinct probability measures (on the real line)  such that for any $P,Q \in \mathcal A$, $\mathbb P\left(d_{\varphi,J}^2(P,Q) = 0\right) \geq 1 -\epsilon $.
\end{proposition}

We are therefore motivated to find  distances of the form (\ref{eq:unsmoothedCharFuncDist}) that
can distinguish larger classes of distributions, yet remain efficient to compute.
These distances are characterized as follows:
\begin{definition}[Random Metric]
\label{rand:metric}
A random process  $d$ with the values in $\mathbf{R}$, indexed with pairs  from the set of probability measures $\mathcal{M}$  
\[
d  = \{ d(P,Q) : P,Q \in \mathcal{M} \} 
\]
is said to be a random metric if it satisfies all the conditions for a metric with qualification `almost surely'. Formally, for all $P,Q,R \in \mathcal{M}$, random variables $d(P,Q),d(P,R),d(R,Q)$ must satisfy 
\begin{enumerate}
 \item $d(P,Q) \geq 0$ a.s.
 \item if $P=Q$, then $d(P,Q)=0$ a.s, if $P \neq Q$ then $d(P,Q)=0$ a.s.
 \item $d(P,Q) = d(Q,P)$ a.s.
 \item $d(P,Q) \leq d(P,R)+d(R,Q)$ a.s. \footnote{ Note that this does not imply that realizations of $d$ are distances on $\mathcal{M}$, but it does imply that they are almost surely distances for all arbitrary finite subsets of $\mathcal{M}$.}
\end{enumerate}

\end{definition}
From the statistical testing point of view, the coincidence axiom of a metric $d$, $d(P,Q)=0$ if and only if $P=Q$, is key, as it  ensures consistency against all alternatives. The quantity $d_{\varphi,J}(P,Q)$ in (\ref{eq:unsmoothedCharFuncDist}) violates the coincidence axiom, so it is only a random pseudometric (other axioms are trivially satisfied).
We remedy this problem by replacing the characteristic functions by smooth characteristic functions:
\begin{definition}
A smooth characteristic function $\phi_{P}(t)$ of a measure $P$ is a characteristic function of $P$ convolved with an analytic smoothing kernel $l$, i.e.
\begin{equation}
\label{eq:defSmoothed}
\phi_{P}(t) = \int_{\mathbf R^d} \varphi_{P}(w) l(t-w) dw,\qquad t\in\mathbf R^d.
\end{equation}
\end{definition}
The analogue of $d_{\varphi,J}(P,Q)$ for smooth characteristic functions is simply
\begin{equation}
d^2_{\phi,J}(P,Q) = \frac{1}{J} \sum_{j=1}^{J} |\phi_P(T_j) - \phi_Q(T_j) |^2, 
\end{equation}
where  $\left\{T_j\right\}_{j=1}^J$ are sampled independently from the absolutely continuous distribution (returning to our earlier example, this might be $F^{-1}\kappa(t)$ if we believe this to be an informative choice).  The following theorem, proved in the Appendix, demonstrates that the smoothing greatly
increases the class of distributions we can distinguish.
\begin{Theorem}
\label{th:charEmb}
Let $l$ be an analytic, integrable kernel with an inverse  Fourier transform strictly greater than zero. Then, for any $J>0$, $d_{\phi,J}$ is a random metric on the space of probability measures with integrable characteristic functions, and $\phi_P$ is an analytic function.
\end{Theorem}
This result is primarily a consequence of analyticity of smooth characteristic functions and the fact that analytic functions are 'well behaved'.
There is an additional, practical advantage to smoothing: when the variability in the difference of the characteristic functions is high, and these differences are local, smoothing distributes the difference in CFs more evenly in the frequency domain (a simple illustration is in Fig. \ref{fig:var}, Appendix), making them easier to find by measurement at a small number of randomly chosen points.
This accounts for the observed improvements in test power in Section \ref{sec:experiments}, over differences in unsmoothed CFs.
\vspace{-2mm}
\paragraph{Metrics based on mean embeddings.}
The key step which led us to the construction of a random metric $d_{\phi,J}$ is the convolution of the original characteristic functions with an analytic smoothing kernel. This idea need not be restricted to the representations of probability measures in the frequency domain. We may instead directly convolve the probability measure  with a positive definite kernel $k$ (that need not be translation invariant), yielding its mean embedding into the associated RKHS,
\begin{equation}
 \mu_{P}(t)= \int_{E} k(x,t) dP(x).
\end{equation}
We say that a positive definite kernel $k:\mathbf{R}^D \times \mathbf{R}^D \to \mathbf{R}$ is analytic on its domain if for all $x \in \mathbf{R}^D$, the feature map $k(x,\cdot)$ is an analytic function on $\mathbf{R}^D$. By using embeddings with \emph{characteristic and analytic} kernels, we obtain particularly useful representations of distributions. As for the smoothed CF case, we define   
\begin{equation}
\label{eq:preStatistic}
d_{\mu,J}^2(P,Q)  = \frac{1}{J} \sum_{j=1}^{J} (\mu_P(T_j) - \mu_Q(T_j) )^2. 
\end{equation}
The following theorem ensures that $d_{\mu,J}(P,Q)$ is also a random metric. 
\begin{Theorem}
\label{th:analyticMean}
Let $k$ be an analytic, integrable and characteristic kernel. Then for any $J>0$, $d_{\mu,J}$ is a random metric on the space of probability measures (and $\mu_P$ is an analytic function).
\end{Theorem}
Note that this result is stronger than the one presented in Theorem \ref{th:charEmb}, since is is not restricted to the class of probability measures with integrable characteristic functions. Indeed, the assumption that the characteristic function is integrable implies the existence and boundedness of a density. Recalling the representation of MMD in \eqref{eq:MMD-meanembedding},
we have proved that it is almost always sufficient to measure difference between  $\mu_P$ and  $\mu_Q$ at a finite number of points, provided our kernel is characteristic and analytic. In the next section, we will see that metrization of the space of probability measures using random metrics $d_{\mu,J}$, $d_{\phi,J}$ is very appealing from the computational point of view. It turns out that  the statistical tests that arise from those metrics have linear time complexity (in the number of samples) and constant memory requirements.

\vspace{-2mm}
\section{Hypothesis Tests Based on Distances Between Analytic Functions}
\label{sec:test}
In this section, we provide two linear-time  two-sample tests: first, a test based on analytic mean embeddings, and  then a test based on  smooth characteristic functions. We further describe the relation with competing alternatives. Proofs of this chapter's propositions are in the Appendix \ref{ap:prrofs}.

\textbf{Difference in analytic functions}
In the previous section we described the random metric based on a difference in analytic mean embeddings, $d_{\mu,J}^2(P,Q)  = \frac{1}{J} \sum_{j=1}^{J} (\mu_P(T_j) - \mu_Q(T_j) )^2.$
If we replace $\mu_P$ with the empirical mean embedding $\hat \mu_P = \frac 1 n \sum_{i=1}^n k(X_i,\cdot)$ it can be shown that for any  sequence  of unique $\{ t_j\}_{j=1}^J$, under the null hypothesis, as $n\to\infty$,
\begin{equation}
\label{eq:sumChi}
\sqrt n \sum_{j=1}^{J} (\hat \mu_{P}(t_j) - \hat \mu_{Q}(t_j))^2
\end{equation}
converges in distribution to a sum of correlated chi-squared variables. Even for fixed $\{ t_j\}_{j=1}^J$, it is very computationally costly to obtain quantiles of this distribution, since this requires a bootstrap or permutation procedure. We will follow a different approach based on Hotelling's $T^2$-statistic  \cite{hotelling1931}.  The Hotelling's $T^2$-squared statistic of a normally distributed, zero mean, Gaussian vector $W = (W_1,\cdots,W_J)$, with a covariance matrix $\Sigma$, is $T^2 = W \Sigma^{-1} W$. The compelling property of the statistic is that it is distributed as a $\chi^2$-random variable with $J$ degrees of freedom. To see a link between $T^2$ and equation (\ref{eq:sumChi}), consider a random variable $\sum_{i=j}^J W_j^2$: this is also distributed as a sum of correlated chi-squared variables. In our case $W$ is replaced with a difference of normalized empirical mean embeddings, and $\Sigma$ is replaced with the empirical covariance of the difference of mean embeddings. Formally, let $Z_i$ denote the vector of differences between kernels at tests points $T_j$,
\begin{equation}
 Z_i = ( k(X_i,T_1) - k(Y_i,T_1), \cdots, k(X_i,T_J) - k(Y_i,T_J) )\in \mathbf R^J.
\end{equation}
We define the vector of mean empirical differences 
$W_n = \frac 1  n \sum_{i=1}^n Z_i, $
and its covariance matrix
$\Sigma_n = \frac 1  n Z Z^{T}$.
The test statistic is
\begin{equation}
 S_n = n W_n \Sigma_n^{-1} W_n.
\end{equation}
The computation of $S_n$ requires inversion of a $J\times J$ matrix $\Sigma_n$, but this is fast and numerically stable: $J$ will typically be small and is in our experiments less than 10. The next proposition demonstrates the use of $S_n$ as a two-sample test statistic.
\begin{proposition}[Asymptotic behavior of $S_n$]
\label{prop:Hotelling}
 Let $d_{\mu,J}^2(P,Q)=0$ a.s. and let $\{X_i\}_{i=1}^n$ and $\{Y_i\}_{i=1}^n$  be i.i.d. samples from $P$ and $Q$ respectively. Then the statistic $S_n$ is a.s. asymptotically distributed as a $\chi^2$-random variable with $J$ degrees of freedom (as $n \to \infty$ with $d$ fixed). If $d_{\mu,J}^2(P,Q)>0$ a.s., then a.s. for any fixed $r$, $\mathbb P(S_n > r) \to 1$  as $n \to \infty$ .
\end{proposition}
We now apply the above proposition in obtaining a statistical test. 

\begin{test}[Analytic mean embedding ]
\label{test}
Calculate $S_n$. Choose a threshold $r_\alpha$ corresponding to the $1-\alpha$ quantile of a  $\chi^2$ distribution with $J$ degrees of freedom, and reject the null hypothesis whenever $S_n$ is larger than $r_\alpha$. 
\end{test}

There are a number of valid sampling schemes for the test points  $\left\{T_j\right\}_{j=1}^J$ to evaluate the differences in mean embeddings: see Section \ref{sec:experiments} for a discussion.

\textbf{Difference in smooth characteristic functions}
From the convolution definition of a smooth characteristic function \eqref{eq:defSmoothed}  it is not clear how to calculate its estimator in linear time. However, we show in the next proposition that a smooth characteristic function can be written as an expected value of some function with respect to the given measure, which can be estimated in a linear time.  
\begin{proposition}
\label{lemma:compChar}
Let $k$ be an integrable translation-invariant kernel and $f$ its inverse Fourier transform. Then the smooth characteristic function of $P$ can be written as  $\phi_{P}(t) = \int_{\mathbf R^d} e^{it^{\top}x} f(x) dP(x).$
\end{proposition}

It is now clear that a test based on the smooth characteristic functions is similar to the test based on mean embeddings. The main difference is in the definition of the vector of differences $Z_i$:
\begin{equation}
 Z_i =  (  f(X_i) \sin(  X_i T_1) - f(Y_i)\sin(Y_i T_1), f(X_i)\cos(  X_i T_1) - f(Y_i) \cos(Y_i T_1), \cdots )\in\mathbf R^{2J}
\end{equation} 
The imaginary and real part of the $e^{\sqrt{-1} T_j^{\top}X_i} f(X_i) - e^{\sqrt{-1} T_j^{\top}Y_i} f(Y_i)$ are stacked together, in order to ensure that $W_n$, $\Sigma_n$ and $S_n$ as all real-valued quantities.  
\begin{proposition}
\label{prop:Hotelling2}
 Let $d_{\mu,J}^2(P,Q)=0$ and let $\{X_i\}_{i=1}^n$ and $\{Y_i\}_{i=1}^n$  be i.i.d. samples from $P$ and $Q$ respectively. Then the statistic $S_n$ is almost surely asymptotically distributed as a $\chi^2$-random variable with $2J$ degrees of freedom (as $n \to \infty$ with $J$ fixed). If $d_{\phi,J}^2(P,Q) >0$ , then almost surely for any fixed $r$, $P(S_n > r) \to 1$ as $n \to \infty$.
\end{proposition}

\textbf{Other tests}.
The test  \cite{EppsSingleton1986} based on empirical characteristic functions was constructed originally for one test point and then generalized to many points - it is quite similar to our second test, but does not perform smoothing (it is also based on a $T^2$-Hotelling statistic). The block MMD \cite{ZarGreBla13} is a sub-quadratic test, which can be trivially linearized by fixing the block size, as presented in the Appendix.   Finally, another alternative is the MMD, an inherently quadratic time test. We scale MMD to linear time by sub-sampling our data set, and choosing only $\sqrt n$ points, so that the MMD complexity becomes $O(n)$. Note, however, that the true complexity of MMD  involves a permutation calculation of the null distribution at cost $O(b_n n)$, where  the number of permutations $b_n$  grows with $n$. See Appendix \ref{sec:otherTests} for a detailed description of alternative tests.  
\vspace{-2mm}
\section{Experiments}
\label{sec:experiments}

In this section we compare two-sample tests on both artificial benchmark data and on real-world data. We denote  the smooth characteristic function test as `Smooth CF', and the test based on the analytic mean embeddings as `Mean Embedding'. We compare against several alternative testing approaches: block MMD (`Block MMD'), a characteristic functions based test (`CF'), a sub-sampling MMD test (`MMD($\sqrt{n}$)'), and the quadratic-time MMD test (`MMD(n)').

\textbf{Experimental setup.} 
For all the experiments, $D$ is the dimensionality of samples in a dataset, $n$ is a number of samples in the  dataset (sample size) and $J$ is number of test frequencies. Parameter selection is required for all the tests. 
The table summarizes the main choices of the parameters made for the experiments. The first parameter is the test function, used to calculate the particular statistic. The scalar $\gamma$ represents the length-scale of the observed data. Notice that for the kernel tests we recover the standard parameterization $\exp( -\|\frac{x}{\gamma}-\frac{y}{\gamma}\|^2) = \exp( -\frac{ \|x-y\|^2}{\gamma^2}) $. The original CF test was proposed without any parameters, hence we added  $\gamma$ to ensure a fair comparison - for this test varying $\gamma$  is equivalent to adjusting the variance of the distribution of frequencies $T_j$. For all tests, the value of the scaling parameter $\gamma$ was chosen so as to maximize test power on a held-out training set: details are described in Appendix \ref{sec:params}. We chose not to optimize the sampling scheme for the Mean Embedding and Smooth CF tests, since this would give them an unfair advantage over the Block MMD, MMD($\sqrt{n}$) and CF tests. The block size in the Block MMD test and the number of test frequencies in the Mean Embedding, Smooth CF, and CF tests, were always set to the same value (not greater than 10) to maintain exactly the same time complexity. Note that we did {\em not} use the popular median heuristic for kernel bandwidth choice (MMD and B-test), since it gives poor  results for the Blobs and AM Audio datasets \cite{GreSriSejetal12}. We do not run MMD(n) test in the 'Simulation 1' and on  the 'Amplitude Modulated Music' since the sample size is $10000$, i.e., too large for a quadratic-time test with permutation sampling for the test critical value.

It is important to verify that Type I error is indeed at the design level, set at $\alpha=0.05$ in this paper. This is verified in the Appendix figure \ref{fig:H0}. Also shown in the plots is the $95\%$ percent confidence intervals for the results, as averaged over 4000 runs.

\begin{center}
  \label{tab:params}
  \begin{tabular}{ | l || c | c | c| }
    \hline
    Test & Test Function & Sampling scheme & Other parameters \\ 
    \hline
    Mean Embedding & $\exp(-\| \frac{x}{\gamma}-t\|^2)$ & $T_j \sim N(0_d,I_d)$ & $J$ - no. of test frequencies\\
    Smooth CF & $\exp( it^{\top}\frac{x}{\gamma} -\|\frac{x}{\gamma}-t \|^2)$ & $T_j \sim N(0_d,I_d)$ & $J$ - no. of test frequencies\\
    MMD(n),MMD($\sqrt{n}$) & $\exp( -\|\frac{x}{\gamma}-\frac{y}{\gamma}\|^2)$ & not applicable & $b$ -bootstraps \\ 
    Block MMD & $\exp( -\|\frac{x}{\gamma}-\frac{y}{\gamma}\|^2)$ & not applicable & $B$-block size \\
    CF & $\exp(it^{\top}\frac{x}{\gamma} )$ & $T_j \sim N(0_d,I_d)$ & $J$ - no. of test frequencies\\
    \hline
  \end{tabular}
\end{center}

\paragraph{Real Data 1: Higgs dataset,}\textit{$D=4$, $n$ varies, $J=10$.}
The first experiment we consider is on the UCI Higgs dataset \cite{Lichman:2013} described in \cite{Baldi2014} - the task is to distinguish signatures of processes which produce Higgs bosons from background processes which do not. We consider a two-sample test on certain extremely  low-level features in the dataset - kinematic properties measured by the particle detectors, i.e., the joint distributions of the azimuthal angular momenta $\varphi$ for four particle jets. We denote by $P$ the jet $\varphi$-momenta distribution of the background process (no Higgs bosons), and by $Q$ the corresponding distribution for the process that produces Higgs bosons (both are distributions on $\mathbf R^4$). As discussed in \cite[Fig. 2]{Baldi2014}, $\varphi$-momenta, unlike transverse momenta $p_T$, carry very little discriminating information for recognizing whether Higgs bosons were produced or not. Therefore, we would like to test the null hypothesis that the distributions of angular momenta $P$ (no Higgs boson observed) and $Q$ (Higgs boson observed) might yet be rejected. The results for different algorithms are presented in the Figure \ref{fig:higgs}. We observe that the joint distribution of the angular momenta is in fact a discriminative feature. Sample size varies from 1000 to 12000.  The Smooth CF test has significantly higher power than the other tests, including the quadratic-time MMD, which  we could only run on up to $5100$ samples  due to computational limitations. The leading performance of the Smooth CF test is especially remarkable given it is several orders of magnitude faster then the quadratic-time MMD(n), which is both expensive to compute, and requires a costly permutation approach to determine the significance threshold. 
\begin{figure}
  \centering    
  \centerline{\includegraphics[width=0.86\textwidth]{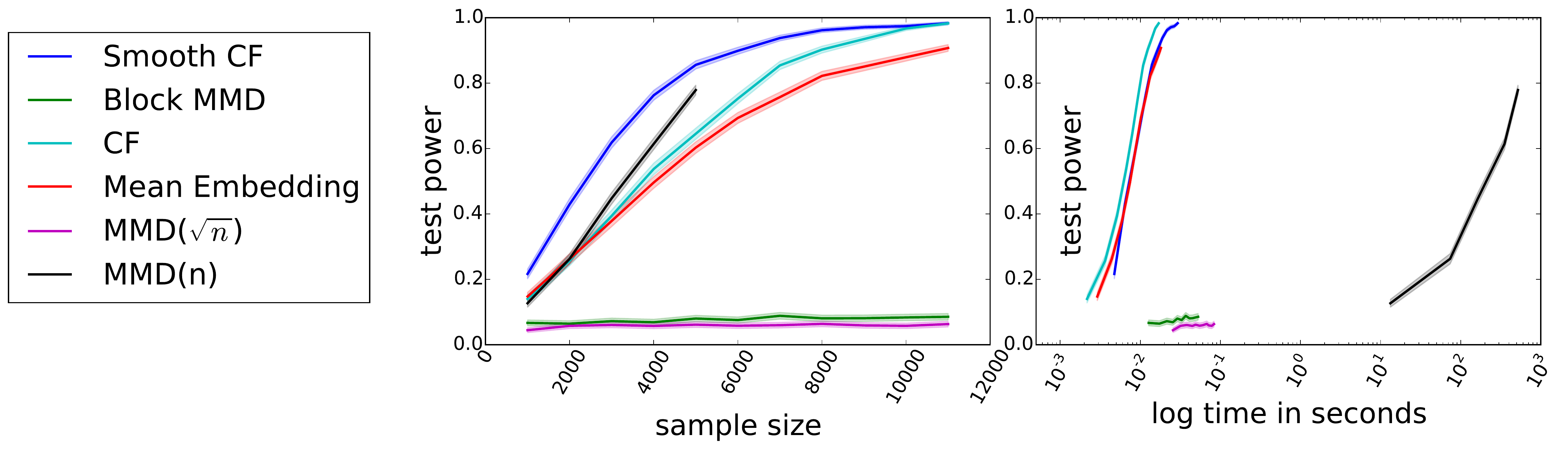}}
  \caption{Higgs dataset.  {\bf Left:} Test power vs.  sample size. {\bf Right:} Test power vs. execution time.}
  \label{fig:higgs}
\end{figure}

\paragraph{Real Data 2: Amplitude Modulated Music,}\textit{$D=1000$, $n=10000$, $J=10$.}
Amplitude modulation is the earliest technique used to transmit voice over the radio. In the following experiment observations were one thousand dimensional samples of carrier signals that were modulated with two different input audio signals from the same album, song $P$ and song $Q$ (further details of these data are described in \cite[Section 5]{GreSriSejetal12}). 
To increase the difficulty of the testing problem, independent Gaussian noise of increasing variance (in the range $1$ to $4.0$) was added to the signals. The results are presented in the Figure \ref{fig:mean}. Compared to the other tests, the Mean Embedding and Smooth CF tests are more robust to the moderate noise contamination.

\begin{figure}
  \centering    
  \centerline{\includegraphics[width=0.86\textwidth]{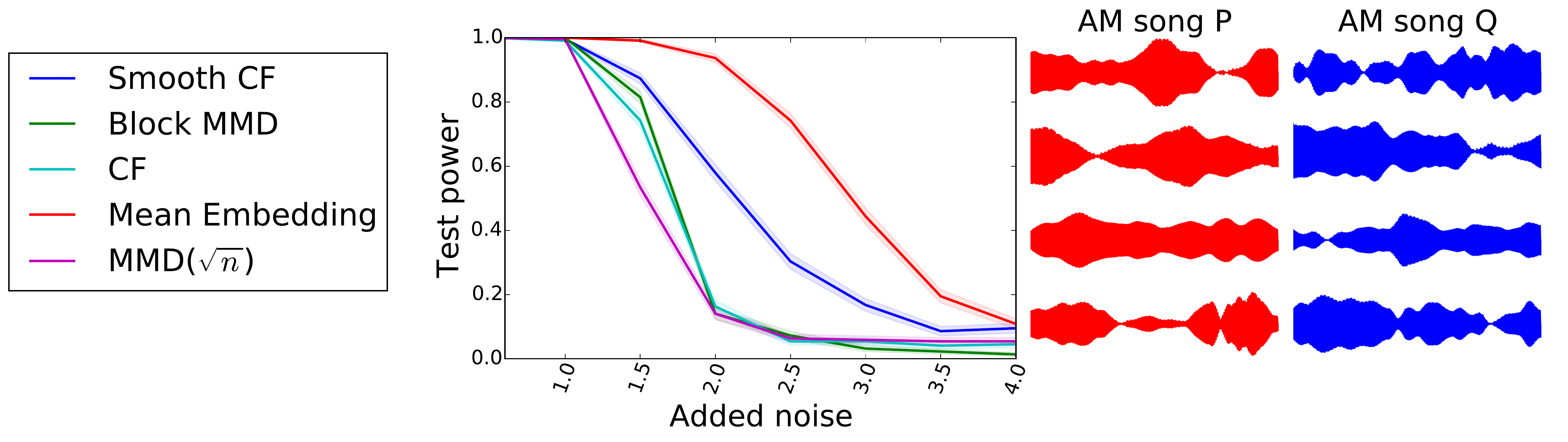}}
  \caption{Music  Dataset.{\bf Left:} Test power vs. added noise. {\bf Right:}  four samples from $P$ and $Q$.}
  \label{fig:mean}
\end{figure}

\paragraph{Simulation 1: High Dimensions,}\textit{$D$ varies, $n=10000$, $J=3$.}
It has been recently shown, in theory and in practice, that the two-sample problem gets more difficult as the number of the dimensions increases on which the distributions do not differ  \cite{RamRed15, RedRam15}.  In the following experiment, we study the power of the two-sample tests as a function of dimension of the samples. We run two-sample test on two datasets of Gaussian random vectors which differ \emph{only} in the first dimension, 
\begin{align*}
 \text{Dataset I: } \quad P = N(0_D,I_D)\qquad vs.& \qquad &Q = N\left((1,0,\cdots,0),I_D\right) \\
 \text{Dataset II: } \quad P = N(0_D,I_D)\qquad vs.& \qquad &Q = N\left(0_D,\text{diag}((2,1,\cdots,1))\right),
\end{align*}
where $0_d$ is a $D$-dimensional vector of zeros, $I_D$ is a $D$-dimensional identity matrix, and $\text{diag}(v)$ is a diagonal matrix with $v$ on the diagonal. The number of dimensions (D) varies from 50 to 1000 (Dataset I) and from 50 to 2500 (Dataset II). The power of the different two-sample tests is presented in  Figure  \ref{fig:varAndMean}. The Mean Embedding test yields best performance for both datasets, where the advantage is especially large for differences in variance.
\begin{figure}
  \centering    
  \centerline{\includegraphics[width=0.86\textwidth]{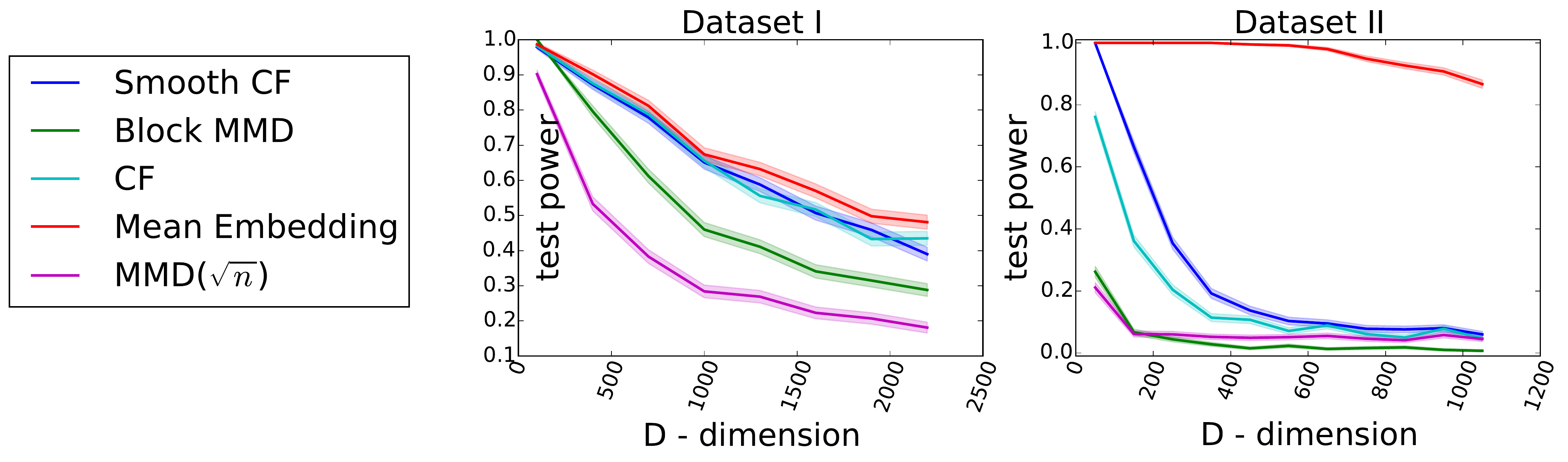}}
  \caption{ Power vs. redundant dimensions comparison for tests on high dimensional data.}
\label{fig:varAndMean}
\end{figure}
\paragraph{Simulation 2: Blobs,}\textit{$D=2$, $n$ varies, $J=5$.}
\label{sec:Blobs}
The Blobs dataset is a grid of two dimensional Gaussian distributions (see Figure \ref{fig:Blobs}), which is known to be a challenging two-sample testing task. The difficulty arises from the fact that the difference in distributions is encoded at a much smaller  lengthscale than the overall data. In this experiment both $P$ and $Q$ are a four by four grid of Gaussians, where $P$ has unit covariance matrix in each mixture component, while each component of $Q$ has a non unit covariance matrix. It was demonstrated by \cite{GreSriSejetal12} that a good choice of  kernel is crucial for this task. Figure \ref{fig:Blobs} presents the results of two-sample tests on the Blobs dataset. The number of samples varies from 50 to 14000 ( MMD(n) reached  test power one with $n=1400$).  We found that the  MMD(n) test has the best power as function of the sample size, but the  worst power/computation tradeoff. By contrast, random distance based tests have the best power/computation tradeoff.    
\begin{figure} 
  \centerline{\includegraphics[width=0.86\textwidth]{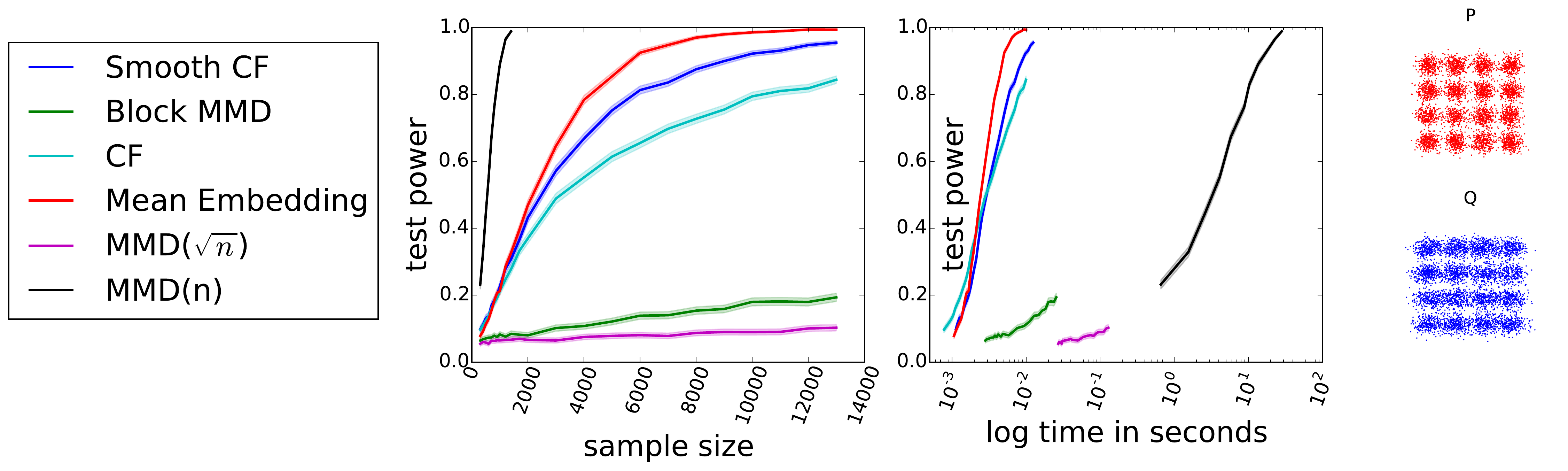}}
  \caption{Blobs Dataset. {\bf Left:}  test power vs. sample size. {\bf Center:} test power vs. execution time. {\bf Right:} illustration of the blob dataset. Each mixture component in the upper plot is a standard Gaussian, whereas those in the lower plot have the direction of the largest variance rotated by $\pi/4$ and amplified so the standard deviation in this direction is 2.}
  \label{fig:Blobs}
\end{figure}

{
\small{
\FloatBarrier
\bibliographystyle{plain}
\bibliography{paper}
}
}

\newpage

\appendix

\setcounter{figure}{0}    
\renewcommand\thefigure{\thesection.\arabic{figure}}

\section{Figures}

\begin{figure}[h!]
  \centering    
  \centerline{\includegraphics[width=0.9\columnwidth]{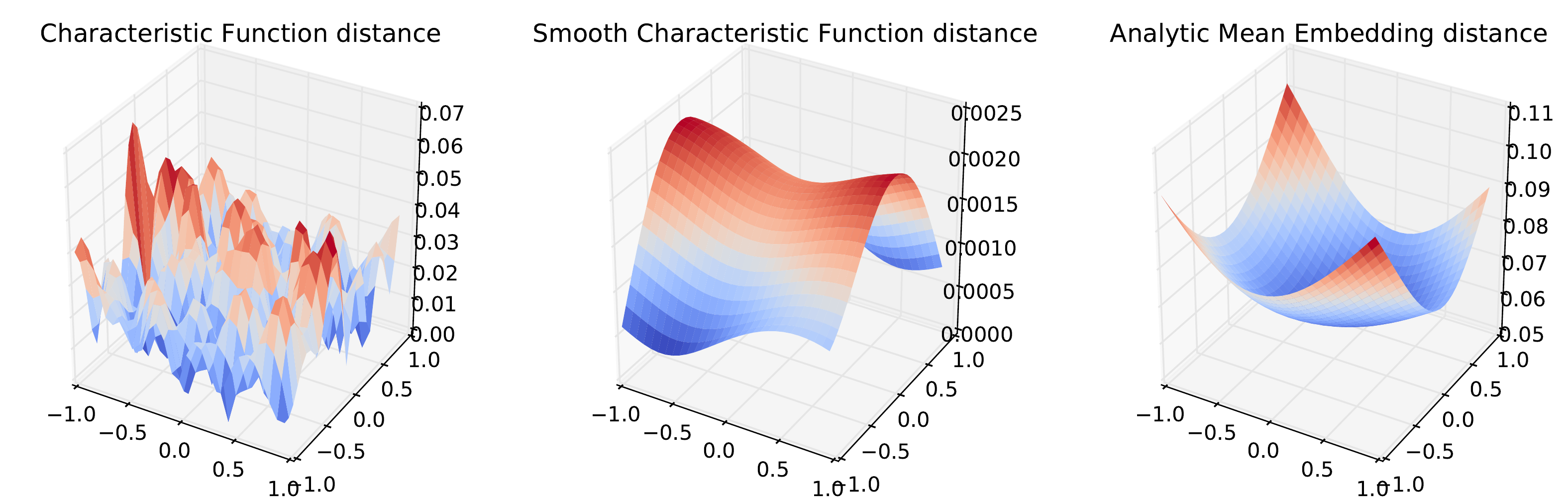}}
  \caption{\textbf{Smooth vs non-smooth.} {\bf Left}: pseudo-distance $d_{\varphi,1}(P,Q)$ which uses a single frequency $t\in\mathbf R^2$ as a function of this frequency; {\bf Middle}: $d_{\phi,1}(P,Q)$ depicted in the same way; {\bf Right}: $d_{\mu,1}(P,Q)$ which uses a single location $t\in\mathbf R^2$ as a function of this location. The measures $P,Q$ used are illustrated in Figure \ref{fig:Blobs} - these are grids of Gaussian distributions discussed in detail in Section \ref{sec:Blobs}.}
  \label{fig:var}
\end{figure}

\begin{figure}[h!]
  \centering    
  \centerline{\includegraphics[width=0.86\textwidth]{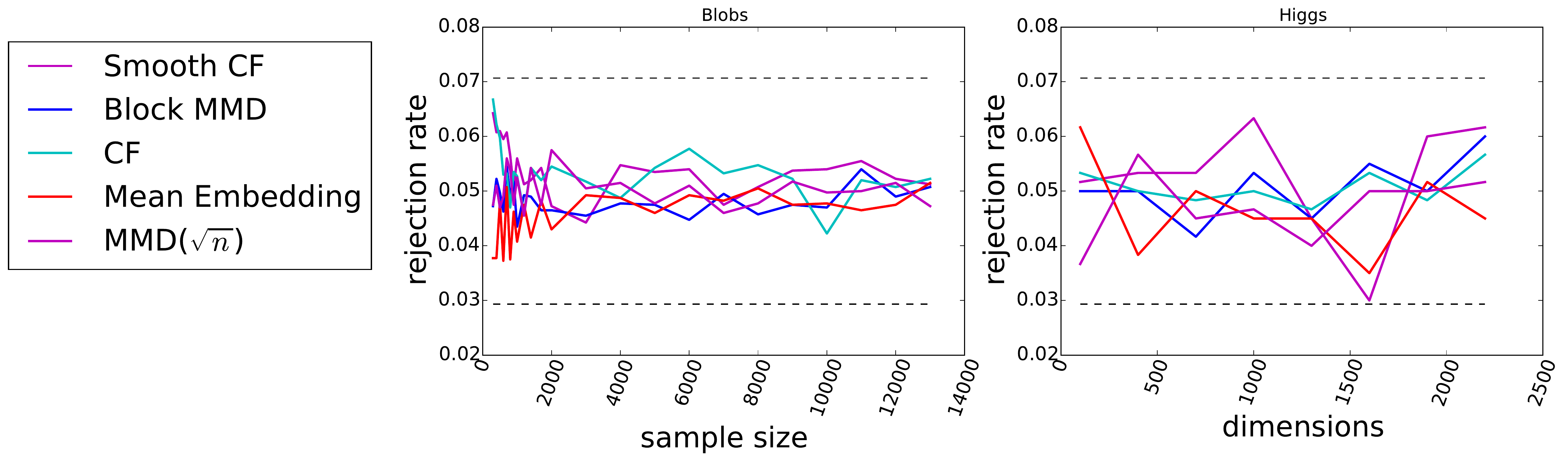}}
  \caption{Type I error of the blobs dataset {\bf (left)}  and the  dimensions dataset {\bf (right)}. 
  The dashed line is the 99\% Wald interval $\alpha\pm 2.57\sqrt{ \alpha(1-\alpha)/ 4000 }$ ($4000$ is number of repetitions) around the design test size of $\alpha=0.05$. } 
  \label{fig:H0}
\end{figure}

\section{Proofs}
\label{ap:prrofs}

\subsection*{Proof of Proposition \ref{prop:notSufficient}}
\begin{proof}
For some $I = I(\epsilon)$, there exists an interval $[-I,I]$ with measure $1-(1-\epsilon)^{\frac 1 J}$. Define $f_w(t) = 1 - w|t| $ for $w > \frac 1 I $ and zero elsewhere. By Polya's theorem, $\mathcal{A} = \{ f_w \}_{w > \frac 1 I}$  is an uncountable family of characteristic functions that are the same on the complement of $[-I,I]$, which has measure $(1-\epsilon)^{\frac 1 J}$. For $w_1 > w_2 > \frac 1 I$, $f_{w_1} \neq f_{w_2}$ in some neighborhood of $1/{w_1}$, hence the measures associated with those characteristic functions are different. The probability that all $T_i$ sit in the complement of interval $[-I,I]$ is $ \left( (1 -\epsilon)^{\frac 1 J} \right)^J = (1 -\epsilon)$ and such an event implies that  $S_{\varphi,J}^2 = 0$.    

\end{proof}

\subsection*{Proof of Theorem \ref{th:analyticMean} }

First we give a proposition that characterizes limits of analytic functions.
 
\begin{proposition}[ {{\cite[Proposition 3]{davidson1983pointwise} }}]
\label{anal}
 If $\{f_n\}$ is a sequence of real valued, uniformly bounded analytic functions on $\mathbf R^d$ converging pointwise to  $f$, then $f$ is analytic.  
\end{proposition}
Now we characterize the RKHS of an analytic kernel. Similar results were proved for specific classes of kernels in  \cite[Theorem 1]{sun2008reproducing}, \cite[Corollary 3.5]{steinwart2006explicit}. 
\begin{lemma}
\label{lem:analyticH}
 If $k$ is a bounded, analytic kernel on $\mathbf R^d \times \mathbf R^d$, then all functions in the RKHS $\mathcal{H}_k$ associated with this kernel are analytic. 
\end{lemma}
\begin{proof}
Since $\mathbf R^d$ is separable then by \cite[Lemma 4.33]{steinwart2008support} Hilbert Space $\mathcal{H}_k$ is separable. By Moore-Aronszajn Theorem  \cite{berlinet2004reproducing} there exist a  set $H_0$ of linear combinations of functions $k(\cdot,x) , x \in \mathbf R^d$, which is dense in $\mathcal{H}_k$ and $\mathcal{H}_k$ is a set of functions which are pointwise limits of Cauchy sequences in $H_0$. For each $f \in \mathcal{H}_k$ let  $\{f_n\} \in \mathcal{H}_0$ be a sequence of functions converging in the Hilbert Space norm to $f$. Since $\{f_n\}$ is convergent there exists $N$ such that $\forall n>N$ $\left\Vert f_n -f \right\Vert \leq 1$. For all $n$ there exist a uniform bound on $f_n$ norm
\begin{align}
 \Vert f_n \Vert = \Vert f_n - f + f \Vert \leq \Vert f_n - f \Vert + \Vert f \Vert 
 \leq \max(1, \max_{1\leq i \leq N} \Vert f_N \Vert ) + \Vert f \Vert.
\end{align}
Since $k$ is bounded, by the \cite[Lemma 4.33]{steinwart2008support}, there exists $C$ such that for any $f \in \mathcal{H}_k$, $\Vert f \Vert_{\infty} \leq C \Vert f \Vert$. Therefore for all $n$
\begin{align}
 \Vert f_n \Vert_{\infty} \leq C \max(1, \max_{1\leq i \leq N} \Vert f_N \Vert ) + C \Vert f \Vert.
\end{align}
Finally, using  Proposition \ref{anal} we conclude that $f$ is analytic. This concludes the proof of Lemma \ref{lem:analyticH}.

\end{proof}

Next, we show that analytic functions are 'well behaved'.
\begin{lemma}
\label{lem:disc}
 Let $\mu$ be absolutely continuous measure on $\mathbf R^d$ (wrt. the Lebesgue measure).  Non-zero, analytic function $f$ can be zero at most at the set of measure 0, with respect to the measure $\mu$.
\end{lemma}
\begin{proof}
If $f$ is zero at the set with a limit point then it is zero everywhere. 
Therefore $f$ can be zero at most at a set $A$ without a limit point, which by definition is a discrete set (distance between any two points in $A$ is greater then some $\epsilon>0$). Discrete sets have zero Lebesgue measure (as a countable union of points with zero measure). Since $P$ is absolutely continuous then $\mu(A)$ is zero as well.
\end{proof}

Next, we show how to construct random distances.
\begin{lemma}
\label{main:lemma}
 Let $\Lambda$ be an injective mapping from the space of the probability measures into a space of analytic functions on $\mathbf R^d$. Define 
 \[
  d^2_{\Lambda,J}(P,Q) = \sum_{j=1}^{J} \Big|\left[\Lambda P\right](T_j) -\left[\Lambda Q\right](T_j) \Big|^2
 \]
where  $\left\{T_j\right\}_{j=1}^J$ are real valued  i.i.d. random variables from a distribution which is absolutely continuous with respect to the Lebesgue measure. Then, $d^2_{\Lambda,J}(P,Q)$ is a random metric.
\end{lemma}
\begin{proof}
Let $\Lambda P$ and $\Lambda Q$ be images of measures $P$ and $Q$ respectively. We want to apply Lemma \ref{lem:disc} to the analytic function $f = \Lambda P - \Lambda Q$, with the measure $\mu = \mu_{T_i}$, to see that if $P \neq Q$ then  $f\neq 0$ a.s. To do so, we need to show that  $P \neq Q$ implies that $f$ is non-zero. Since mapping to $\Lambda$ is injective, there must exists at least one point $o$ where $f$ is non-zero. By continuity of $f$, there exists a ball around $o$ in which $f$ is non-zero.

We have shown that  $P \neq Q$ implies  $f\neq 0$ a.s. which in turn implies that  $d_{\Lambda,J}(P,Q) > 0$ a.s. If $P = Q$ then $f=0$ and $d_{\Lambda,J}(P,Q) = 0$. 

By the construction $d_{\Lambda,J}(P,Q) = d_{\Lambda,J}(Q,P)$ and for any measure $U$, $d_{\Lambda,J}(P,Q) \leq d_{\Lambda,J}(P,U) +d_{\Lambda,J}(U,Q)$ a.s. since the triangle inequality holds for any vectors in $\mathbf R^J$. 
\end{proof}

We are ready to proof Theorem \ref{th:analyticMean}.
\begin{proof}[Proof of Theorem \ref{th:analyticMean}]
 Since $k$ is characteristic the mapping $\Lambda : P \to \mu_P$ is injective. Since $k$ is a bounded, analytic kernel on $\mathbf R^d \times \mathbf R^d$, the Lemma \ref{lem:analyticH} guarantees that $\mu_P$ is analytic, hence the image of $\Lambda$ is a subset of analytic functions. Therefore, we can use Lemma \ref{main:lemma} to see that $ d_{\Lambda,J}(P,Q)^2 =  d_{\mu,J}(P,Q)^2$ is a random metric. This concludes the proof of Theorem \ref{th:analyticMean}.
\end{proof}

\subsection*{Proof of Theorem \ref{th:charEmb}}
We first show that smooth characteristic functions are unique to distributions. 
\begin{lemma}
\label{lem:meanEmb}
 If $l$ is an analytic, integrable, translation invariant  kernel with an inverse  Fourier transform strictly greater then zero and $P$ has integrable characteristic function, then the mapping
 \[
  \Lambda: P \to \phi_P 
 \]
is injective and $\phi_P $ is element of the RKHS $\mathcal{H}_l$ associated with $l$.
\end{lemma}

\begin{proof}
For the integrable characteristic function $\varphi$ we define a functional   $L : \mathcal{H}_l \to R$ given by formula 
\begin{equation}
 Lf = \int_{\mathbf R^d} \varphi(x) f(x) dx
\end{equation}
Since $L(f+g) = L(f) + L(g)$, $L$ is linear. We check that $L$ is bounded; let $B = \{ f \in \mathcal{H}_l : \parallel f \parallel \leq 1 \}$ be a unit ball in the Hilbert Space.   
\begin{align}
 \sup_{ f \in B}  |Lf|  \leq \sup_{ f \in B} \int_{\mathbf R^d} \varphi(x) f(x) dx \leq \sup_{ f \in B} \int_{\mathbf R^d} \varphi(x) \| f \| l(x,x) dx  = \int_{\mathbf R^d} \varphi(x) l(x,x) dx \leq \infty
\end{align}
By Riesz representation Theorem there exist $\phi \in H$ such that $\langle \phi, f \rangle =  \int_{\mathbf R^d} \varphi(x) f(x) dx$. By reproducing property $\phi$ is given by equation $\phi(x) = \langle \phi, l(t,)\rangle = \int_{\mathbf R^d} l(x,t) \varphi(x) dx$.
With each probability measure $P$ with an integrable characteristic function $\varphi_P$  we associate the smooth characteristic function with 
\begin{equation}
 P \to \phi_P(x) = \int_{\mathbf R^d} l(x,t) \varphi_P(x) dx
\end{equation}

We will prove that $P \to \phi_P$ is injective. We will show that , $\forall_x \phi_Q(x) = \phi_P(x)$ implies $P=Q$.  
\begin{equation}
\label{th:eq}
 \phi_Q = \phi_P \Rightarrow \int_{\mathbf R^d} l(x -t) \varphi_P(x) dx= \int_{\mathbf R^d} l(x -t) \varphi_Q(x) dx.
\end{equation}
We apply inverse Fourier transform to this convolution and get 
\begin{equation}
g(x) f_{X}(x) = f_{Y}(x) g(x)
\end{equation}
Where $g = T^{-1}l$, $f_{Y} = T^{-1}\varphi_Q$ and $f_{X} = T^{-1}\varphi_P$. Since inverse Fourier transform is injective on the space of the  integrable characteristic functions, and all $l, \varphi_P, \varphi_Q$ are integrable CFs, then application of the inverse Fourier transform does not enlarge the null space of Eq.  \eqref{th:eq}. Since $g(x)>0$, $f_{X}(x) = f_{Y}(x)$ everywhere, implying that the mapping $P \to \phi_P$ is injective. This concludes the proof of Lemma \ref{lem:meanEmb}.

\end{proof}
Next, we show that smooth characteristic function is analytic.
\begin{lemma}
\label{lem:analyticHH}
 If $l$ is an analytic, integrable kernel with an inverse  Fourier transform strictly greater then zero and $P$ has an integrable characteristic function then the smooth characteristic function  $\phi_P$ is analytic. 
\end{lemma}

\begin{proof}
 By lemma \ref{lemma:compChar}, all functions in the  RKHS associated with $l$ are analytic and by \ref{lem:meanEmb} $\phi_P$ is an element of this RKHS.
\end{proof}

We are ready to proof Theorem \ref{th:charEmb}.

\begin{proof}[Proof of Theorem \ref{th:charEmb}]
 Since $l$ is an analytic, integrable kernel with an inverse  Fourier transform strictly greater then zero then by the Lemma \ref{lem:meanEmb} the mapping $\Lambda : P \to \phi_P$ is injective and $\Lambda(P)$ is an element of the RKHS associated with $l$.   The Lemma \ref{lem:analyticHH} shows that $\mu_P$ is analytic. Therefore we can use Lemma \ref{main:lemma} to see that $ d_{\Lambda,J}(P,Q)^2 =  d_{\phi,J}(P,Q)^2$ is a random metric.
This concludes the proof of Theorem \ref{th:charEmb}
\end{proof}

\paragraph{Proof of Lemma \ref{lemma:compChar}}

\begin{proof}
By  Fubini's theorem we get 
\begin{align*}
\phi_{P}(t) &=  \int_{\mathbf R^d} \varphi_{P}(t-w) f(w) dw\\
&=  \int_{\mathbf R^d} \left( \int_{\mathbf R^d} e^{i(t-w)^{\top}x}  dP(x) \right) f(w) dw\\
&= \int_{\mathbf R^d} e^{it^{\top}x} \left( \int_{\mathbf R^d} e^{-iw^{\top}x} f(w) dw \right) dP(x) \\
&= \ev [e^{it^\top X} F f(X)].
\end{align*}
Use of Fubini's theorem is justified, since the iterated integral is finite \cite{rudin1987real}[Theorem 8.8 b] i.e.
\begin{align*}
 \int_{\mathbf R^d}&  \int_{\mathbf R^d} |e^{i(t-w)^{\top}x}  f(w)|dP(x)  dw \\
 &= \int_{\mathbf R^d} |f(w)|  \int_{\mathbf R^d} 1 dP(x) dw < \infty.
\end{align*}

\end{proof}

\paragraph{Proof of Proposition \ref{prop:Hotelling}  }

\begin{proof}
The probability space of random variables $\{ T_j \}_{1 \leq j \leq J}$ and $\{ X_i\}_{1 \leq i \leq n}$ is a product space i.e sequence of  $T_j$'s is defined on the space $(\Omega_1, \mathcal F_1 , P_1)$ and the sequence of $X_i$'s is defined on the space $(\Omega_2, \mathcal F_2 , P_2)$. We will show that for almost all $\omega \in \Omega_1$, $S_n$ converges to $\chi^2$ distribution with $J$ degrees of freedom. We define 
\begin{equation}
 Z_i^{\omega} = ( k(X_i,T_1({\omega})) - k(Y_i,T_1(\omega)), \cdots, k(X_i,T_J(\omega)) - k(Y_i,T_J(\omega)) )\in \mathbf \mathbf R^J.
\end{equation}
If there exist $a \neq b$, such that $T_a(\omega) = T_b(\omega)$, then we set  $Z_i^{\omega}=0$. Otherwise, if  $\ev  Z_i^{\omega} = 0$ then $\sqrt n W_n^{\omega} = \frac 1  n \sum_{i=1}^n Z_i^{\omega}$ converges to a multivariate Gaussian vector with  covariance matrix $\Sigma^{\omega} = \ev Z_i^{\omega} (Z_i^{\omega})^{T}$ (the variance $Z_i^{\omega} $ is finite so we use standard multivariate CLT). Therefore $\lim_{n \to \infty} n (W_n^{\omega})^{\top} (\Sigma_n^{-1})^{\omega}  W_n^{\omega}$  has asymptotically $\chi^2$ distribution with $J$ degrees of freedom (by the CLT and Slutsky's theorem). Consider
\begin{equation}
 d_{\mu,J}^{\omega}(P,Q)  = \left( \frac{1}{J} \sum_{j=1}^{J} \left| \mu_P(T_j(\omega)) - \mu_Q(T_j(\omega))  \right| ^2 \right)^{0.5}. 
\end{equation}
If $d_{\mu,J}^{\omega}(P,Q) = 0$ then for all $j$, $(\mu_P(T_j(\omega)) = \mu_Q(T_j(\omega))$, which implies that $\ev Z_i=0$.

If $\ev Z_i^{\omega} \neq 0$ then  
\begin{equation}
P(S_n^{\omega}  > r) = P\left((W_n^{\omega})^{\top} (\Sigma_n^{-1})^{\omega}   W_n^{\omega} - \frac r n > 0 \right) \to 1.
\end{equation}
To see that, first we show that $(\Sigma_n^{-1})^{\omega}$ converges in probability to the positive definite  matrix  $(\Sigma^{-1})^{\omega}$. Indeed, each entry of the matrix $\Sigma_n^{\omega}$ converges to the matrix $\Sigma^{\omega}$, hence entires of the matrix $(\Sigma^{-1})^{\omega}$, given by a continuous function of the entries of $\Sigma^{\omega}$, are limit of the sequence $(\Sigma_n^{-1})^{\omega}$. Similarly $ W_n^{\omega}$  converges in probability to the vector $W^{\omega}$. Since $(W^{\omega})^{\top} (\Sigma^{-1})^{\omega}   W^{\omega} = a^{\omega} >0$ ($(\Sigma^{-1})^{\omega}$ is positive definite), then $(W_n^{\omega})^{\top} (\Sigma_n^{-1})^{\omega}   W_n^{\omega} $, being a continuous function of  the entries of $W_n^{\omega}$ and $ (\Sigma_n^{-1})^{\omega}$,  converges to $a^{\omega}$. On the other hand $\frac r n$ converges to zero and the proposition follows. Finally since $d_{\mu,J}^{\omega}(P,Q) >0$ almost surely then $\ev Z_i^{\omega} \neq 0$ for almost all $\omega \in \Omega_1$.

We have showed that the proposition hold for almost all $\omega$. Indeed it does not hold if it happens that for some $a \neq b$,  $T_a(\omega) = T_b(\omega)$ or $d_{\mu,J}^{\omega}(P,Q)=0$ for $P \neq Q$. But both those events have zero measure.

\end{proof}

\paragraph{Proof of Proposition \ref{prop:Hotelling2} }
The poof is analogue to the proof of the Proposition \ref{prop:Hotelling}.

\section{Other tests}
\label{sec:otherTests}

\subsection{Quadratic-time MMD test}
\label{sec:mmd}
For two measures $P$, $Q$ the population $MMD$ can be written as 
\[
 MMD(P,Q)^2 = \int k(x,x') dP(x) dP(x') - 2 \int k(x,y) dP(x) dP(y) + \int k(y,y') dP(y) dP(y').
\]

An MMD-based test uses as its statistic an empirical estimator of the squared population MMD, and rejects the null if this is larger than a threshold $r_\alpha$ corresponding to the $1-\alpha$ quantile of the null distribution. The minimum variance unbiased estimator of MMD is
\begin{align*}
MMD_n^2
&=
\frac{1}{\binom{n}{2}}\sum_{i \neq j} h(X_i,X_j,Y_i,Y_j),\\
 h(x,x',y,y') &= k(x,x') + k(y,y') - k(x,y')-k(x',y). 
\end{align*}
The test threshold $r_\alpha$ is costly to compute. The null distribution of $MMD^2_n$ is an infinite weighted sum of chi-squared random variables, where the weights are  eigenvalues of the kernel with respect to the (unknown) distribution $P$. A bootstrap or permutation procedure may be used in obtaining consistent quantiles of the null distribution, however the cost is  $O(b_n n^2)$ if we have $b_n$ permutations and $n$ data points ($b_n$ is usually in the hundreds, at minimum). As an alternative consistent procedure, the eigenvalues of the joint gram matrix over samples from $P$ and $Q$ may be used in place of the population eigenvalues; the fastest quadratic-time MMD test uses a gamma approximation to the null distribution, which works well most of the times, but has no consistency guarantees \cite{GreFukHarSri09}. 

\subsection{Sub-quadratic time MMD test}\label{sec:alt}

An alternative to the quadratic-time MMD test is a B-test (block-based test): the idea is to break the data into blocks, compute a quadratic-time statistic
 on each block, and average  these quantities to obtain the test statistic.
More specifically, for an individual block, laying on the main diagonal and starting at position $(i-1)B+1$, the  statistic $\eta(i)$ is calculated  as
\begin{equation}
\eta(i) = \frac{1}{\binom{B}{2}} \sum^{iB}_{a=(i-1)B+1} \sum^{iB}_{b=(i-1)B+1 \neq a} h(X_a,X_b,Y_a,Y_b).
\end{equation}
The overall test statistic is then
\begin{equation}\label{eq:b-statistic}
  \eta = \frac B n \sum_{i=1}^{\frac n B} \eta(i).
\end{equation}


 The choice of $B$ determines computation time - at one extreme is the linear-time MMD suggested by \cite{Gretton2012,GreSriSejetal12} where we have $n/2$ blocks of size $B=2$, and at the other extreme is the usual full MMD with $1$ block of size $n$, which requires calculating the test statistic on the whole kernel matrix in quadratic time. 
In our case, the size of the block remains constant as $n$ increases, and is greater than 2. This is very similar to the case proposed by \cite{ZarGreBla13}, and the consistency of the test is not affected. 

B-test of \cite{ZarGreBla13} assumes that $B\to\infty$ together with $n$, which implies that the statistic $\hat\eta$ defined in \eqref{eq:b-statistic} under the null distribution satisfies
\begin{equation}\label{eq:btest_null}
 \sqrt{nB}\hat\eta \overset{D}{\to} \mathcal N\left(0, 4\sigma_0^2\right),
\end{equation}
for asymptotic variance $\sigma_0^2=\mathbb{E}_{XX'}k^{2}(X,X')+\left(\mathbb{E}_{XX'}k(X,X')\right)^{2}-2\mathbb{E}_{X}\left[\left(\mathbb{E}_{X'}k(X,X')\right)^{2}\right]$ that can easily be estimated directly or by considering the empirical variance of the statistics computed within each of the blocks. Note that the same asymptotic variance $\sigma_0^2$ is obtained in the case of a quadratic-time statistic \cite{Gretton2012} -- albeit convergence rate being a faster $O(1/n)$ in that case. Indeed, \eqref{eq:btest_null} is obtained directly from the leading term of the variance of each block-based statistic being $\frac{4\sigma_0^2}{B^2}$. Therefore, the p-value for B-test is approximated as $\Phi\left(-\frac{\sqrt{nB}\hat\eta}{2\hat\sigma_0}\right)$, where $\Phi$ is the standard normal cdf. When $B$ remains constant as $n$ increases, it can be shown that the variance of each block-based statistic is exactly $\frac{4\sigma_0^2}{B(B-1)}$, and thus we obtain by CLT that
\begin{equation*}
 \sqrt{n}\hat\eta \overset{D}{\to} \mathcal N\left(0, \frac{4\sigma_0^2}{B-1}\right).
\end{equation*}
Therefore, a slight change to p-value needs to be applied when $\sigma_0^2$ is estimated directly: $\Phi\left(-\frac{\sqrt{n(B-1)}\hat\eta}{2\hat\sigma_0}\right)$. If, however, one simply uses the empirical variance of the individual statistics computed within each block, the procedure is unaffected.

\section{Parameters Choice}
\label{sec:params}

We split our data set into two disjoint sets, training and testing set, and optimize parameters on the training set. We didn't come up with an automated testing procedure, instead we plotted the p-values of tests for different scales. The figure $\ref{fig:choice}$ presents such a plot for three different tests. The p-values were obtained by running the test several times (20 to 50) for each data scaling $\lambda$. Note that in the case of simulations we just generated new training dataset for each repetition for a given data scaling. For the music dataset we generated new noises for each scaling and for the Higgs dataset we have used bootstrap. The last method is applicable to real life problems i.e. we split our data into training and test part and then bootstrap from the training part.   
\begin{figure*}
\label{fig:choice}
  \centering    
  \centerline{\includegraphics[width=\textwidth]{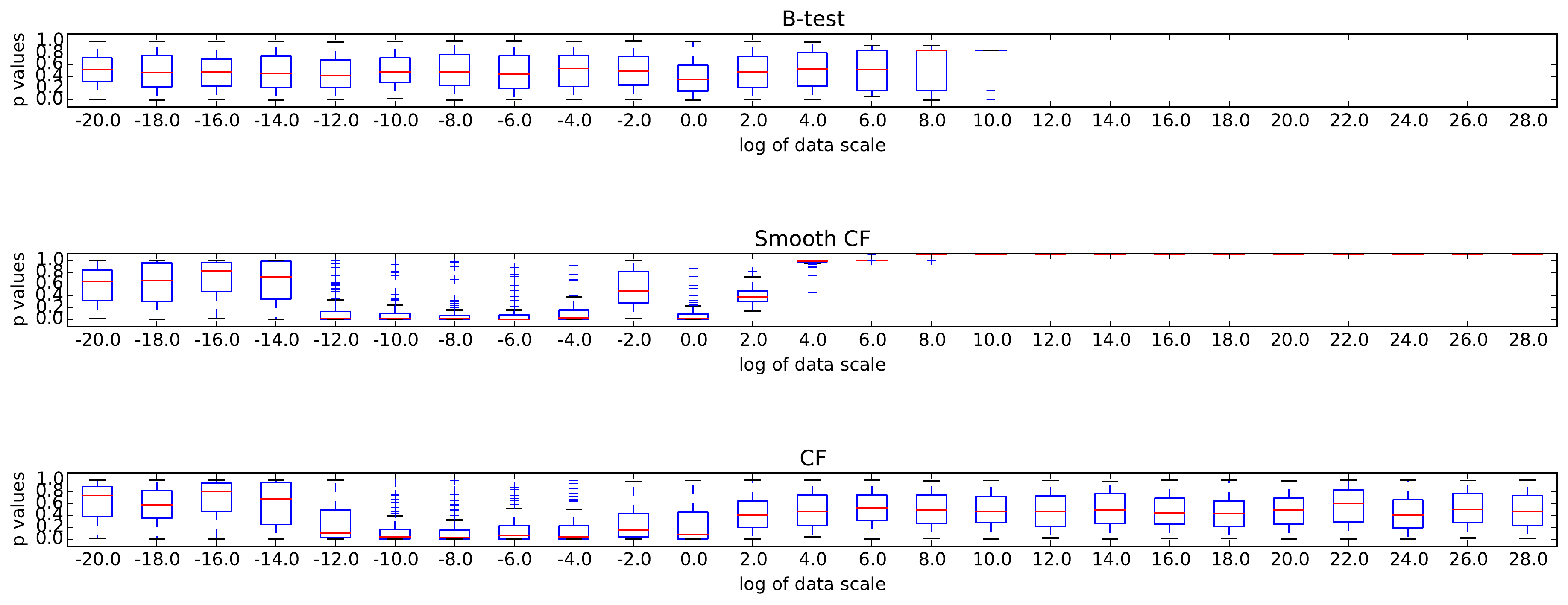}}
  \caption{Box plot of p-values used for parameter selection. The $X$ axis shows the binary logarithm of the scaling parameter applied to data. We have chosen the scaling with the smallest median. If the medians were similar we have used the one that had less outliers and was surrounded with other scalings with small p-value. In the example we have chosen $2^0.0$ scaling for the B-test, $2^{-8.0}$ scaling for the Smoothed CF and $2^{-10.0}$ scaling for the CF test. }
\end{figure*}

\end{document}